\documentclass[runningheads]{llncs}
\usepackage[T1]{fontenc}
\usepackage{graphicx}
\usepackage{booktabs}
\usepackage{amsmath}
\usepackage{amssymb}
\usepackage{xcolor}
\usepackage{hyperref}
\usepackage{cleveref}
\usepackage{listings}
\usepackage{xspace}
\newtheorem{assumption}{Assumption}
\newcommand{\llbracket}{[\![}
\newcommand{\rrbracket}{]\!]}

\lstdefinelanguage{turtle}{
  morekeywords={@prefix, a},
  morestring=[b]",
  morecomment=[l]{\#},
  sensitive=true
}
\begin{document}
\title{Denotational Semantics for ODRL: 
       Knowledge-Based Constraint Conflict Detection}
\titlerunning{Denotational Semantics for ODRL Conflict Detection}

\author{
  Daham M. Mustafa\inst{1,2} \and
  Diego Collarana\inst{2} \and
  Yixin Peng\inst{1} \and
  Rafiqul Haque\inst{3} \and
  Christoph Lange-Bever\inst{1,2} \and
  Christoph Quix\inst{1,2} \and
  Stefan Decker\inst{1,2}
}
\authorrunning{D. Mustafa et al.}

\institute{
  RWTH Aachen University, Germany \and
  Fraunhofer FIT, Sankt Augustin, Germany \and
  University of Galway, Ireland \\
  \email{daham.mustafa@rwth-aachen.de}
}
\date{}
\maketitle
\begin{abstract}
ODRL's six set-based operators---\texttt{isA}, \texttt{isPartOf}, 
\texttt{hasPart}, \texttt{isAnyOf}, \texttt{isAllOf}, 
\texttt{isNoneOf}---depend on external domain knowledge that 
the W3C specification leaves unspecified. Without it, every 
cross-dataspace policy comparison defaults to \textsc{Unknown}.
We present a denotational semantics that maps each ODRL 
constraint to the set of knowledge-base concepts satisfying 
it. Conflict detection reduces to denotation intersection 
under a three-valued verdict---\textsc{Conflict}, 
\textsc{Compatible}, or \textsc{Unknown}---that is sound 
under incomplete knowledge. The framework covers all three 
ODRL composition modes (\texttt{and}, \texttt{or}, 
\texttt{xone}) and all three semantic domains arising in 
practice: taxonomic (class subsumption), mereological 
(part-whole containment), and nominal (identity). For 
cross-dataspace interoperability, we define order-preserving 
alignments between knowledge bases and prove two guarantees: 
conflicts are preserved across different KB standards, and 
unmapped concepts degrade gracefully to 
\textsc{Unknown}---never to false conflicts. A runtime 
soundness theorem ensures that design-time verdicts hold 
for all execution contexts.
The encoding stays within the decidable EPR fragment of 
first-order logic. We validate it with 154~benchmarks 
across six knowledge base families (GeoNames, ISO~3166, 
W3C DPV, a GDPR-derived taxonomy, BCP~47, and ISO~639-3) 
and four structural KBs targeting adversarial edge cases. 
Both the Vampire theorem prover and the Z3 SMT solver 
agree on all 154~verdicts. A key finding is that exclusive 
composition (\texttt{xone}) requires strictly stronger KB 
axioms than conjunction or disjunction: open-world semantics 
blocks exclusivity even when positive evidence appears to 
satisfy exactly one branch.
\end{abstract}



\section{Introduction}
\label{sec:introduction}

The Open Digital Rights Language (ODRL) is the W3C Recommendation 
for expressing digital policies~\cite{iannella2018odrl}, adopted 
across European data ecosystems including Gaia-X~\cite{gaiax2023architecture}, IDSA and the 
Eclipse Dataspace Connector~\cite{idsa2022connector}. ODRL defines set-based 
operators---\texttt{isA}, \texttt{isPartOf}, \texttt{isAnyOf}, 
\texttt{isNoneOf}---but provides no mechanism to \emph{compute} 
them. Their meaning depends on external domain knowledge that 
the standard deliberately leaves unspecified.

Table~\ref{tab:challenges} traces a scenario from the German 
cultural dataspace (Datenraum Kultur)~\cite{ToubekisDecker2025} through four escalating 
interoperability challenges. The Bayerische Staatsbibliothek (BSB)~\cite{bsb-preisliste}
permits access to its digitized German language manuscripts for 
non-commercial use within Europe. A French national archive 
requests access from France, in French, for scientific research.

\textbf{Challenge~1: Single-KB grounding.} Without a knowledge 
base, an ODRL engine sees incomparable strings and defaults to 
\textsc{Unknown}. Grounding against GeoNames, BCP~47, and W3C~DPV 
yields per-dimension verdicts: \textsc{Compatible} (France 
$\preceq$ Europe), \textsc{Conflict} (French 
$\not\sqsubseteq$ German), \textsc{Unknown} (DPV classifies 
scientific research under R\&D without committing to the 
commercial/non-commercial distinction).

\textbf{Challenge~2: Cross-dataspace composition.} Three 
constraint dimensions must overlap simultaneously under 
conjunction. A single \textsc{Conflict} blocks the request; 
the framework identifies which dimension caused the failure.

\textbf{Challenge~3: Cross-standard alignment.} The BSB uses 
GeoNames~\cite{geonames2024} and DPV~\cite{pandit2024dpv}; the 
French archive uses ISO~3166~\cite{iso3166} and a GDPR-derived 
taxonomy~\cite{gdpr2016}. When a concept lacks a counterpart 
(GeoNames has \texttt{bavaria}; ISO~3166 does not), the 
verdict degrades to \textsc{Unknown}---never to a false 
\textsc{Conflict}. Conflicts detected in one KB must be 
preserved across aligned KBs.

\begin{table*}[t]
\centering
\caption{Four interoperability challenges illustrated through 
a single scenario: BSB permits European access to German-language 
manuscripts for non-commercial use; a French archive requests 
access from France, in French, for scientific research.}

\label{tab:challenges}
\small
\begin{tabular}{@{}p{2.4cm}p{5.6cm}p{4.0cm}p{3.2cm}@{}}
\toprule
\textbf{Challenge} & \textbf{Scenario} & \textbf{Without framework} & \textbf{With framework} \\
\midrule
\textbf{Single-KB grounding}
\newline\textit{Can we compute ODRL's set operators?}
  & Policy uses \texttt{isPartOf}, \texttt{isA}; request uses 
    \texttt{eq}. Three dimensions:
    \newline\texttt{spatial}: France vs.\ Europe
    \newline\texttt{purpose}: ScientificRes.\ vs.\ NonCommercial
    \newline\texttt{language}: fr vs.\ de
  & All three return \textsc{Unknown}---engine sees incomparable 
    strings
  & Per-dimension verdicts:
    \newline\textsc{Compatible} (France $\preceq$ Europe)
    \newline\textsc{Unknown} (R\&D $\sqsubseteq$? NonComm.)
    \newline\textsc{Conflict} (fr $\not\sqsubseteq$ de) \\
\addlinespace[6pt]
\textbf{Cross-dataspace composition}
\newline\textit{Do multi-dimensional policies compose?}
  & BSB policy: 3 constraints. French request: 3 constraints.
    Compatibility requires \emph{all} pairs to overlap.
  & No mechanism to conjoin verdicts across operands
  & Language \textsc{Conflict} blocks entire request. 
    Diagnostic isolates blocking dimension. \\
\addlinespace[6pt]
\textbf{Cross-standard alignment}
\newline\textit{What if each side uses a different KB?}
  & BSB uses GeoNames, DPV, BCP\,47. French archive uses 
    ISO\,3166, GDPR taxonomy, ISO\,639-3. Same scenario, 
    different standards.
  & Policies from different standards are mutually 
    opaque---no interoperability
  & Conflicts preserved across KB pairs. Unmapped concepts 
    degrade to \textsc{Unknown}, never to false 
    \textsc{Conflict}. \\
\bottomrule
\end{tabular}
\end{table*}

None of these challenges is addressed by existing approaches. 
Current ODRL reasoners either ignore set-based operators or 
hardcode domain-specific logic that cannot generalize across 
dataspaces. The result is a ``default deny'' posture that 
blocks legitimate data sharing.

\textbf{Challenge~4: Policy quality assurance.} Policies must 
be internally consistent before data flows. Self-contradictions 
arise when \texttt{and} is written instead of \texttt{or}, 
requiring a purpose to be simultaneously commercial and 
academic---unsatisfiable. Redundancies occur when constraints 
imply each other: requiring France \emph{and} Europe adds no 
restriction since France $\preceq$ Europe. The DSSC blueprint 
further requires downstream policies to only \emph{restrict} 
upstream terms~\cite{dssc2024}. Without formal verification, 
these errors go undetected.

We present a \emph{semantic grounding framework} that addresses 
all four challenges. Each ODRL constraint maps to a 
\emph{denotation}---the set of KB concepts satisfying it. 
Conflict detection reduces to denotation intersection: 
non-empty means \textsc{Compatible}; provably empty means 
\textsc{Conflict}; insufficient KB information means 
\textsc{Unknown}. For cross-standard alignment, we define 
order-preserving injections between concept spaces and prove 
that conflicts are preserved while unmapped concepts degrade 
gracefully. Design-time verdicts guarantee runtime enforcement.

The encoding stays within the Effectively Propositional (EPR) 
fragment of first-order logic, guaranteeing decidability. 
Bidirectional denotation rules are essential: \emph{if}-direction 
axioms populate denotations from KB facts (proving compatibility); 
\emph{only-if} axioms constrain membership (proving conflict). 
Neither direction alone suffices.

ODRL defines 14 KB-dependent left operands across three semantic 
domains: taxonomic ($\sqsubseteq$: purpose, language), 
mereological ($\preceq$: spatial), and nominal (identity: device). 
A single framework handles all 14 uniformly.

\smallskip\noindent\textbf{Contributions.}
(1)~A \textbf{denotational semantics} for all eight ODRL 
KB-dependent operators across three semantic domains, supporting 
conflict detection, self-contradiction analysis, redundancy 
identification, and refinement verification---all under a 
three-valued verdict sound under incomplete knowledge 
(Section~\ref{sec:framework}).
(2)~A \textbf{decidable EPR encoding} with bidirectional 
denotation rules covering all three ODRL composition modes, 
validated by 100\% Vampire/Z3 agreement across 154~benchmarks 
(Section~\ref{sec:evaluation}).
(3)~A \textbf{benchmark suite} of 154~problems across six KB 
families and four structural KBs, revealing that \texttt{xone} 
requires strictly stronger KB axioms than \texttt{and}/\texttt{or} 
(Section~\ref{sec:evaluation}).

This framework operates at the \emph{constraint level}: it 
determines whether constraint denotations can co-exist, not 
whether a complete ODRL rule activates. The W3C ODRL Formal 
Semantics draft~\cite{odrl-formal-semantics} specifies rule 
activation via constraint satisfaction, duty fulfillment, and 
refinement matching. Our framework addresses the logically prior 
question: \emph{can two constraints ever be simultaneously 
satisfied?} Static conflict detection prunes unsatisfiable 
combinations at design time; operational semantics evaluates 
activation at runtime. The two are complementary.

\paragraph{Supplementary Material.}
Related work and empirical runtime analysis will appear in 
the extended version. The benchmark suite, Isabelle/HOL 
theory files, and TPTP/SMT-LIB encodings will be made 
available at \url{https://github.com/Daham-Mustaf/odrl-benchmark} 
upon publication.
\section{Semantic Grounding Framework}
\label{sec:framework}

We formalize the grounding of ODRL constraints against external 
knowledge bases. Conflict detection reduces to testing denotation 
intersection under a three-valued semantics that is sound even 
when knowledge is incomplete.

\begin{definition}[ODRL Constraint]
\label{def:constraint}
A constraint is a tuple $c = (\ell, \bowtie, v)$ where $\ell \in 
\mathcal{L}$ is a left operand, $\bowtie \in \mathcal{O}$ is an 
operator, and $v \in \mathcal{V}$ is a right operand value.
\end{definition}

\begin{definition}[Knowledge Base]
\label{def:kb}
A knowledge base for operand $\ell$ is a structure 
$\mathcal{KB}_\ell = (C, \leq, \perp\!\!\!\perp, \gamma)$ where:
\begin{itemize}
\item $C$ is a finite set of domain concepts,
\item $\leq$ is a reflexive, transitive relation on $C$,
\item $\perp\!\!\!\perp \;\subseteq C \times C$ is a symmetric, 
      irreflexive disjointness relation satisfying 
      \emph{downward closure}: if $x \perp\!\!\!\perp y$, 
      $x' \leq x$, and $y' \leq y$, then 
      $x' \perp\!\!\!\perp y'$,
\item $\gamma : \mathcal{V} \rightharpoonup C$ maps right operand 
      values to concepts ($\gamma(v) = \bot$ when unmapped).
\end{itemize}
The finiteness of $C$ ensures decidability of conflict detection. 
The framework requires only $\leq$ and $\perp\!\!\!\perp$; the 
following classification guides KB selection in practice:
\begin{description}
\item[Taxonomic:] $\leq \,{=}\, \sqsubseteq$ (class subsumption). 
      Operands: \texttt{purpose}, \texttt{language}, \texttt{industry}, 
      \texttt{fileFormat}, \texttt{media}, \texttt{recipient}.
\item[Mereological:] $\leq \,{=}\, \preceq$ (part-whole containment). 
      Operands: \texttt{spatial}, \texttt{virtualLocation}.
\item[Nominal:] $\leq \,{=}\, {=}$ (identity only); 
      $\perp\!\!\!\perp$ is total on distinct concepts 
      ($x \perp\!\!\!\perp y \iff x \neq y$). 
      Operands: \texttt{deliveryChannel}, \texttt{device}, \texttt{event}, 
      \texttt{product}, \texttt{system}, \texttt{unitOfCount}. 
      Under identity, \texttt{isA} degenerates to \texttt{eq} and 
      \texttt{isAllOf} requires all values to coincide.   
\end{description}
\noindent The classification is a deployment guideline; all 
definitions are parameterized by $\leq$ and 
$\perp\!\!\!\perp$ alone.
\end{definition}

\begin{lemma}[Disjointness--Order Consistency]
\label{lem:disj-consistency}
If $x \leq y$ then $\neg(x \perp\!\!\!\perp y)$.
\end{lemma}
\begin{proof}
Suppose $x \leq y$ and $x \perp\!\!\!\perp y$. By reflexivity, 
$x \leq x$. Downward closure with $x' = y' = x$ gives 
$x \perp\!\!\!\perp x$, contradicting irreflexivity.
\end{proof}

\begin{assumption}[KB Correctness]
\label{assm:kb-correctness}
For grounded values ($\gamma(v) \neq \bot$), KB relations---both 
$\leq$ and $\perp\!\!\!\perp$---accurately model domain semantics. 
Knowledge bases may be incomplete but contain no incorrect positive 
($\leq$) or negative ($\perp\!\!\!\perp$) assertions.
\end{assumption}

\begin{table}[t]
\centering
\caption{KB-dependent ODRL operands with semantic domain classification}
\label{tab:kb-operands}
\small
\begin{tabular}{lllll}
\toprule
\textbf{Left Operand} & \textbf{ODRL Specification} & \textbf{Domain} & \textbf{Relation} & \textbf{$\perp\!\!\!\perp$ Source} \\
\midrule
\texttt{language}        & ``A natural language''              & Taxonomic     & $\sqsubseteq$ & Ontology \\
\texttt{purpose}         & ``A defined purpose''               & Taxonomic     & $\sqsubseteq$ & Ontology \\
\texttt{industry}        & ``A defined industry sector''       & Taxonomic     & $\sqsubseteq$ & Ontology \\
\texttt{fileFormat}      & ``A transformed file format''       & Taxonomic     & $\sqsubseteq$ & Ontology \\
\texttt{media}           & ``Category of a media asset''       & Taxonomic     & $\sqsubseteq$ & Ontology \\
\texttt{recipient}       & ``The party receiving the result''  & Taxonomic     & $\sqsubseteq$ & Ontology \\
\addlinespace[3pt]
\texttt{spatial}         & ``A named geospatial area''         & Mereological  & $\preceq$ & Siblings \\
\texttt{virtualLocation} & ``A location of the IT space''      & Mereological  & $\preceq$ & Siblings \\
\addlinespace[3pt]
\texttt{deliveryChannel} & ``The delivery channel used''       & Nominal       & ---       & $x \neq y$ \\
\texttt{device}          & ``An identified device''            & Nominal       & ---       & $x \neq y$ \\
\texttt{event}           & ``An identified event''             & Nominal       & ---       & $x \neq y$ \\
\texttt{product}         & ``Category of product or service''  & Nominal       & ---       & $x \neq y$ \\
\texttt{system}          & ``An identified computing system''  & Nominal       & ---       & $x \neq y$ \\
\texttt{unitOfCount}     & ``The unit of measure''             & Nominal       & ---       & $x \neq y$ \\
\bottomrule
\end{tabular}
\end{table}

\subsection{Constraint Interpretation}
\label{sec:interpretation}
Each constraint is mapped to a \emph{denotation}: the set of domain 
concepts satisfying it. When the KB cannot ground a value, the 
denotation is $\top$ (indeterminate, i.e., not computable from the 
KB---distinct from the universal concept in description logics).

\begin{definition}[Constraint Denotation]
\label{def:denotation}
Let $\mathcal{KB}_\ell = (C, \leq, \perp\!\!\!\perp, \gamma)$ be 
the knowledge base for operand $\ell$, and let 
$c = (\ell, \bowtie, v)$ with $g = \gamma(v)$. If $g = \bot$, 
then $\llbracket c \rrbracket_{\mathcal{KB}} = \top$. Otherwise:
\smallskip
\noindent\textit{Equality operators (all domains):}
\begin{align}
\llbracket \ell \texttt{ eq } v \rrbracket_{\mathcal{KB}} &= \{g\} \\
\llbracket \ell \texttt{ neq } v \rrbracket_{\mathcal{KB}} &= C \setminus \{g\}
\end{align}
\noindent\textit{Hierarchical operators:}
\begin{align}
\llbracket \ell \texttt{ isA } v \rrbracket_{\mathcal{KB}} &= 
  \{x \in C \mid x \leq g\} \\
\llbracket \ell \texttt{ isPartOf } v \rrbracket_{\mathcal{KB}} &= 
  \{x \in C \mid x \leq g\}
\end{align}
\texttt{isA} and \texttt{isPartOf} receive identical 
denotations because the semantic distinction is carried by 
the KB's ordering $\leq$, not the operator 
label\footnote{The ODRL vocabulary glosses \texttt{isA} as 
instance-of and \texttt{isPartOf} as contained-by, but 
provides no distinct formal evaluation rules---both test 
downward closure in a hierarchy. Our framework recovers the 
intended distinction through KB selection: taxonomic KBs 
supply $\sqsubseteq$, mereological KBs supply $\preceq$ 
(Definition~\ref{def:kb}). The operator label is thus 
redundant given the KB's semantic domain.} 
(cf.\ Definition~\ref{def:kb}).
\begin{align}
\llbracket \ell \texttt{ hasPart } v \rrbracket_{\mathcal{KB}} &= 
  \{x \in C \mid g \leq x\}
\end{align}
\noindent\textit{Set-valued operators} (with $g_i = \gamma(v_i)$; 
if any $g_i = \bot$, then $\llbracket c \rrbracket_{\mathcal{KB}} = \top$):
\begin{align}
\llbracket \ell \texttt{ isAnyOf } \{v_1, \ldots, v_n\} \rrbracket_{\mathcal{KB}} &= 
  \textstyle\bigcup_{i=1}^{n} \{x \in C \mid x \leq g_i\} \\
\llbracket \ell \texttt{ isAllOf } \{v_1, \ldots, v_n\} \rrbracket_{\mathcal{KB}} &= 
  \{x \in C \mid \forall i : x \leq g_i\} \\
\llbracket \ell \texttt{ isNoneOf } \{v_1, \ldots, v_n\} \rrbracket_{\mathcal{KB}} &= 
  C \setminus \textstyle\bigcup_{i=1}^{n} \{x \in C \mid x \leq g_i\}
\end{align}
\noindent In all cases, $\leq$ is the ordering from 
$\mathcal{KB}_\ell$ (Definition~\ref{def:kb}).
\end{definition}

\begin{definition}[Conservative Intersection]
\label{def:intersection}
For denotations $D_1, D_2 \in \mathcal{P}(C) \cup \{\top\}$:
\begin{equation}
D_1 \sqcap D_2 = 
\begin{cases}
\emptyset & \text{if } D_1 = \emptyset \text{ or } D_2 = \emptyset \\
D_1 \cap D_2 & \text{if } D_1, D_2 \neq \top \\
\top & \text{otherwise}
\end{cases}
\end{equation}
The first case ensures that a provably empty denotation yields 
$\emptyset$ regardless of the other operand, since 
$\emptyset \cap S = \emptyset$ for any set $S$---including 
indeterminate sets. Here $\top$ denotes epistemic indeterminacy---the 
denotation cannot be computed from the KB---not the universal concept 
set $C$. The domain $\mathcal{P}(C) \cup \{\top\}$ forms a flat 
information ordering where $\top$ sits above all classical sets; 
$\sqcap$ is monotone with respect to this ordering.
\end{definition}


\begin{definition}[Conflict Detection]
\label{def:conflict}
For constraints $c_1, c_2$ over the same operand $\ell$:
\begin{equation}
\text{verdict}(c_1, c_2) = 
\begin{cases}
\textsc{Conflict}    & \text{if } \llbracket c_1 \rrbracket_{\mathcal{KB}} \sqcap 
                              \llbracket c_2 \rrbracket_{\mathcal{KB}} = \emptyset \\
\textsc{Compatible}  & \text{if } \llbracket c_1 \rrbracket_{\mathcal{KB}} \sqcap 
                              \llbracket c_2 \rrbracket_{\mathcal{KB}} \notin 
                              \{\emptyset, \top\} \\
\textsc{Unknown}     & \text{if } \llbracket c_1 \rrbracket_{\mathcal{KB}} \sqcap 
                              \llbracket c_2 \rrbracket_{\mathcal{KB}} = \top
\end{cases}
\end{equation}
\noindent When $\gamma$ is total on all values appearing in a 
policy set, no denotation equals $\top$ and every verdict is 
definite (\textsc{Conflict} or \textsc{Compatible}).
\end{definition}

\begin{theorem}[Soundness]
\label{thm:soundness}
If $\text{verdict}(c_1, c_2) = \textsc{Conflict}$, then no value 
simultaneously satisfies both constraints.
\end{theorem}
\begin{proof}
\textsc{Conflict} requires $\llbracket c_1 \rrbracket \sqcap 
\llbracket c_2 \rrbracket = \emptyset$. Two cases arise:

\textit{Case 1:} Both denotations are grounded (not $\top$) with 
$\llbracket c_1 \rrbracket \cap \llbracket c_2 \rrbracket = \emptyset$. 
By Assumption~\ref{assm:kb-correctness}, no concept exists in both 
denotations.

\textit{Case 2:} One denotation is $\emptyset$ (the other may be $\top$). 
Without loss of generality, let $\llbracket c_1 \rrbracket = \emptyset$. 
By Definition~\ref{def:denotation}, no KB concept satisfies $c_1$. 
By Assumption~\ref{assm:kb-correctness}, this is accurate: no runtime 
value can satisfy $c_1$, hence no value satisfies both constraints.
\end{proof}

\begin{assumption}[Operand Independence]
\label{assm:operand-independence}
KB axioms are operand-local: no axiom relates concepts across 
different operands. Each operand $\ell$ is interpreted over its 
own concept space $C_\ell$, and witnesses for different operands 
are drawn from disjoint domains. This justifies decomposing 
multi-operand verdicts into independent per-operand tests.
\end{assumption}

\begin{definition}[Constraint Composition]
\label{def:composition}
Given two ODRL rules (e.g., a permission and a prohibition) 
whose refinement constraints share operands 
$\ell_1, \ldots, \ell_n$, each operand yields a constraint 
pair $(c_1^{(k)}, c_2^{(k)})$---one from each rule. ODRL 
groups these pairs using three modes. Let 
$V_k = \text{verdict}(c_1^{(k)}, c_2^{(k)})$ denote the 
per-operand verdict from Definition~\ref{def:conflict}.
\smallskip
\noindent\textit{Conjunction} (\texttt{and}): all operand pairs 
must overlap simultaneously.
\begin{equation}
\text{verdict}_{\texttt{and}} = 
\begin{cases}
\textsc{Compatible} & \text{if } \forall k: V_k = \textsc{Compatible} \\
\textsc{Conflict}   & \text{if } \exists k: V_k = \textsc{Conflict} \\
\textsc{Unknown}    & \text{otherwise}
\end{cases}
\end{equation}
\noindent\textit{Disjunction} (\texttt{or}): at least one operand 
pair must overlap.
\begin{equation}
\text{verdict}_{\texttt{or}} = 
\begin{cases}
\textsc{Compatible} & \text{if } \exists k: V_k = \textsc{Compatible} \\
\textsc{Conflict}   & \text{if } \forall k: V_k = \textsc{Conflict} \\
\textsc{Unknown}    & \text{otherwise}
\end{cases}
\end{equation}
\noindent\textit{Exclusive disjunction} (\texttt{xone}): exactly 
one operand pair must overlap, and exclusivity must be provable.
\begin{equation}
\text{verdict}_{\texttt{xone}} = 
\begin{cases}
\textsc{Compatible} & \text{if } \exists! k: V_k = \textsc{Compatible} \\
                     & \quad\text{and } \forall j \neq k: V_j = \textsc{Conflict} \\
\textsc{Conflict}   & \text{if } \forall k: V_k = \textsc{Conflict} \\
\textsc{Unknown}    & \text{otherwise}
\end{cases}
\end{equation}
The composition follows three-valued Kleene semantics: 
\textsc{Unknown} propagates conservatively through conjunction 
and exclusive disjunction, while disjunction can resolve it if 
one branch yields a definite verdict. Conjunction and disjunction 
require only positive knowledge to reach definite verdicts. 
Exclusive disjunction additionally requires \emph{provable 
non-overlap} of the other branch---via $\perp\!\!\!\perp$ axioms 
in the KB, contradictory operators, or structural impossibility. 
For example, if $g_1 \perp\!\!\!\perp g_2$ from the KB's 
disjointness relation, the non-overlapping branch is provably 
empty by downward closure (Definition~\ref{def:kb}). 
When such evidence is absent, open-world 
semantics forces \textsc{Unknown}, even if positive evidence 
appears to satisfy exactly one branch 
(cf.\ ODRL086 in Section~\ref{sec:evaluation}).
\end{definition}

\begin{remark}[Choice of Three-Valued Semantics]
\label{rem:three-valued}
We adopt strong Kleene semantics ($K_3$) because \textsc{Unknown} 
represents epistemic incompleteness, not logical indeterminacy. 
Three requirements guide the choice: (i)~soundness: \textsc{Conflict} 
implies unsatisfiability under all KB completions; (ii)~monotonicity: 
adding KB facts may refine \textsc{Unknown} but never reverse definite 
verdicts; (iii)~compositionality: verdicts combine locally without 
quantifying over completions. Kleene satisfies all three. 
Supervaluations require reasoning over all completions---incompatible 
with decidable EPR. Bochvar is too conservative 
($T \lor \bot = \bot$). Paraconsistent logics address inconsistency, 
not incompleteness.
\end{remark}\begin{remark}[Scope of Composition Operators]
\label{rem:composition-scope}
The ODRL specification defines \texttt{and}, \texttt{or}, and 
\texttt{xone} as \emph{intra-rule} connectives determining a 
single rule's activation~\cite{iannella2018odrl}. Our framework 
applies these same connectives to \emph{cross-rule} conflict 
detection, testing whether constraints from different rules 
(e.g., a permission and a prohibition) can be simultaneously 
satisfied.
\end{remark}
\begin{theorem}[Composition Soundness]
\label{thm:composition-soundness}
If $\operatorname{verdict}_{\texttt{and}}$ yields \textsc{Conflict}, 
then no execution context simultaneously satisfies all constraint 
pairs. Analogously for \texttt{or} and \texttt{xone}.
\end{theorem}
\begin{proof}
\textit{Conjunction.} \textsc{Conflict} requires 
$\exists k: V_k = \textsc{Conflict}$. By 
Theorem~\ref{thm:soundness}, no value satisfies both 
$c_1^{(k)}$ and $c_2^{(k)}$. By 
Assumption~\ref{assm:operand-independence}, operands are 
interpreted over disjoint concept spaces, so satisfying other 
operand pairs cannot provide a witness for operand $\ell_k$. 
The conjunction fails.

\textit{Disjunction.} \textsc{Conflict} requires 
$\forall k: V_k = \textsc{Conflict}$. By 
Theorem~\ref{thm:soundness} applied to each $k$, no operand 
pair admits a satisfying value. By 
Assumption~\ref{assm:operand-independence}, witnesses for 
distinct operands are drawn from disjoint domains, so no 
disjunct succeeds.

\textit{Exclusive disjunction.} \textsc{Conflict} requires 
$\forall k: V_k = \textsc{Conflict}$, reducing to the 
disjunction case.
\end{proof}

\begin{proposition}[Decidability]
\label{prop:decidability}
For any finite $\mathcal{KB}_\ell$ and grounded constraints 
$c_1, c_2$, $\operatorname{verdict}(c_1, c_2)$ is computable.
\end{proposition}
\begin{proof}
Finiteness of $C$ ensures each denotation 
(Definition~\ref{def:denotation}) is a finite, enumerable 
subset of $C$. Set intersection and emptiness testing over 
finite sets are decidable.
\end{proof}

\begin{table}[t]
\centering
\caption{Conflict detection verdicts and their prover encodings. 
The negated-conjecture pattern maps each verdict to a unique 
SZS/SMT-LIB result pair.}
\label{tab:verdicts}
\small
\begin{tabular}{lllll}
\toprule
\textbf{$\llbracket c_1 \rrbracket \sqcap \llbracket c_2 \rrbracket$} 
  & \textbf{Verdict} 
  & \textbf{Conjecture} 
  & \textbf{SZS}~\cite{sutcliffe2008szs}
  & \textbf{SMT} \\
\midrule
$\emptyset$        & \textsc{Conflict}   & $\neg\exists x(\ldots)$ & Theorem        & \texttt{unsat} \\
$S \neq \emptyset$ & \textsc{Compatible} & $\exists x(\ldots)$     & Theorem        & \texttt{unsat} \\
$\top$             & \textsc{Unknown}    & either                  & CounterSat     & \texttt{sat} \\
\bottomrule
\end{tabular}
\end{table}

\begin{definition}[Constraint Subsumption]
\label{def:subsumption}
For constraints $c_1, c_2$ over the same operand $\ell$:
\begin{equation}
\text{subsumes}(c_1, c_2) = 
\begin{cases}
\textsc{Confirmed} & \text{if } \llbracket c_1 \rrbracket, 
                      \llbracket c_2 \rrbracket \neq \top 
                      \text{ and } 
                      \llbracket c_1 \rrbracket \subseteq 
                      \llbracket c_2 \rrbracket \\
\textsc{Refuted}   & \text{if } \llbracket c_1 \rrbracket, 
                      \llbracket c_2 \rrbracket \neq \top 
                      \text{ and } 
                      \llbracket c_1 \rrbracket \not\subseteq 
                      \llbracket c_2 \rrbracket \\
\textsc{Unknown}   & \text{otherwise}
\end{cases}
\end{equation}
When $\text{subsumes}(c_1, c_2) = \textsc{Confirmed}$, we write 
$c_1 \sqsubseteq_c c_2$ and say $c_1$ \emph{refines} $c_2$: 
every value satisfying $c_1$ also satisfies $c_2$, so $c_1$ is 
at least as restrictive. Over grounded denotations, $\sqsubseteq_c$ 
is a preorder; it becomes a partial order on the quotient 
$\mathcal{C}/{\equiv}$ where $c_1 \equiv c_2$ iff 
$\llbracket c_1 \rrbracket = \llbracket c_2 \rrbracket$.
\end{definition}

\begin{lemma}[Conflict Propagation]
\label{lem:conflict-propagation}
If $c_1 \sqsubseteq_c c_2$ and 
$\operatorname{verdict}(c_2, c_3) = \textsc{Conflict}$, then 
$\operatorname{verdict}(c_1, c_3) = \textsc{Conflict}$.
\end{lemma}
\begin{proof}
$c_1 \sqsubseteq_c c_2$ gives 
$\llbracket c_1 \rrbracket \subseteq \llbracket c_2 \rrbracket$ 
with both grounded. \textsc{Conflict} gives 
$\llbracket c_2 \rrbracket \cap \llbracket c_3 \rrbracket = 
\emptyset$ with $\llbracket c_3 \rrbracket$ grounded. Since 
$\llbracket c_1 \rrbracket \subseteq \llbracket c_2 \rrbracket$, 
we have $\llbracket c_1 \rrbracket \cap 
\llbracket c_3 \rrbracket = \emptyset$, and all three denotations 
are grounded, yielding \textsc{Conflict}.
\end{proof}

\subsection{Cross-Dataspace Alignment}
\label{sec:alignment}
The preceding framework assumes a shared KB per operand. In practice, 
dataspaces maintain independent knowledge bases for the same 
domain--GeoNames~\cite{geonames2024} vs.\ ISO~3166~\cite{iso3166} 
for spatial concepts, BCP~47~\cite{bcp47} vs.\ ISO~639-3~\cite{iso639} for language codes, or national purpose taxonomies derived 
from different regulations. A conflict detected in one KB must remain 
valid when translated to another; conversely, translation must never 
fabricate conflicts where none exist. We formalize these guarantees 
through order-preserving alignment.

\begin{definition}[KB Alignment]
\label{def:alignment}
An alignment from $\mathcal{KB}_A = (C_A, \leq_A, 
\perp\!\!\!\perp_A, \gamma_A)$ to $\mathcal{KB}_B = (C_B, \leq_B, 
\perp\!\!\!\perp_B, \gamma_B)$ is an injective partial function 
$\alpha: C_A \rightharpoonup C_B$ satisfying: for all 
$x, y \in \text{dom}(\alpha)$,
\begin{align}
x \leq_A y &\iff \alpha(x) \leq_B \alpha(y) \\
x \perp\!\!\!\perp_A y &\implies 
  \alpha(x) \perp\!\!\!\perp_B \alpha(y)
\end{align}
\noindent Additionally, if $\gamma_A(v) = g \neq \bot$ and 
$\alpha(g) \neq \bot$, then 
${\downarrow}\!g \subseteq \operatorname{dom}(\alpha)$ 
(\emph{witness completeness}).

Injectivity ensures distinct concepts remain distinct. The 
biconditional preserves hierarchical structure in both directions, 
preventing false compatibilities. Disjointness preservation 
($\Rightarrow$ only) retains source incompatibilities. Witness 
completeness ensures all concepts below a mapped grounding value 
are themselves mapped, preventing false conflicts through witness 
loss (Example~\ref{ex:witness-loss}). We write $\alpha(v) = \bot$ 
when $v \notin \operatorname{dom}(\alpha)$.
\end{definition}

\begin{example}[Witness Loss without Witness Completeness]
\label{ex:witness-loss}
Let $C_A = \{a, b, c\}$ with $a \leq b$, $a \leq c$, $b$ and $c$ 
incomparable. For $c_1 = (\ell, \texttt{isA}, v_b)$ and 
$c_2 = (\ell, \texttt{isA}, v_c)$: 
$\llbracket c_1 \rrbracket_A = \{a,b\}$, 
$\llbracket c_2 \rrbracket_A = \{a,c\}$, 
intersection $= \{a\}$, verdict \textsc{Compatible}. 
An alignment $\alpha$ with $\operatorname{dom}(\alpha) = \{b,c\}$ 
satisfies order and disjointness preservation vacuously, but 
$a \notin \operatorname{dom}(\alpha)$. The aligned denotations 
$\{\alpha(b)\}$ and $\{\alpha(c)\}$ have empty intersection, 
fabricating \textsc{Conflict}. Witness completeness prevents this: 
$\alpha(b) \neq \bot$ requires ${\downarrow}\!b = \{a,b\} \subseteq 
\operatorname{dom}(\alpha)$.
\end{example}

Three alignment configurations arise, each with distinct guarantees:
\begin{description}
\item[Total] ($\text{dom}(\alpha) = C_A$). All concepts map. 
  Typical when KBs follow the same standard with different 
  identifiers (e.g., BCP~47 $\leftrightarrow$ ISO~639-3). 
  Verdicts are fully preserved.
\item[Partial] ($\text{dom}(\alpha) \subset C_A$). Some concepts 
  lack counterparts. Typical when KBs differ in granularity 
  (e.g., GeoNames includes cities and regions; ISO~3166 covers 
  only countries). Unmapped concepts degrade verdicts toward 
  \textsc{Unknown} but never introduce false conflicts.
\item[Empty] ($\text{dom}(\alpha) = \emptyset$). KBs are 
  structurally incompatible. All cross-KB verdicts default to 
  \textsc{Unknown}---the status quo without alignment.
\end{description}

\begin{definition}[Aligned Constraint]
\label{def:constraint-alignment}
Given alignment $\alpha: C_A \rightharpoonup C_B$ and constraint 
$c = (\ell, \bowtie, v)$ in $\mathcal{KB}_A$ with 
$g = \gamma_A(v)$:
\begin{equation}
\alpha(c) = 
\begin{cases}
(\ell, \bowtie, v') \text{ with } \gamma'_B(v') = \alpha(g) 
  & \text{if } g \neq \bot \text{ and } \alpha(g) \neq \bot \\
\top & \text{otherwise}
\end{cases}
\end{equation}
When alignment succeeds, $\alpha(c)$ is evaluated over the restricted $\mathcal{KB}'_B = (\alpha(C_A), \leq_B \restriction_{\alpha(C_A)}, \perp\!\!\!\perp_B \restriction_{\alpha(C_A)}, \gamma'_B)$

When it fails, the aligned constraint has indeterminate 
denotation.
\end{definition}

\begin{lemma}[Denotation Preservation]
\label{lem:denotation-equality}
If $\alpha$ is an alignment with 
$\llbracket c \rrbracket_A \subseteq \text{dom}(\alpha)$, then 
interpretation commutes with alignment:
\begin{equation}
\llbracket \alpha(c) \rrbracket_{\mathcal{KB}'_B} = 
\alpha(\llbracket c \rrbracket_A)
\end{equation}
\end{lemma}
\begin{proof}
For $c = (\ell, \texttt{isA}, v)$ with $g = \gamma_A(v)$:
($\subseteq$): Let $y \in \llbracket \alpha(c) 
\rrbracket_{\mathcal{KB}'_B}$. Then $y \in \alpha(C_A)$ and 
$y \leq_B \alpha(g)$. Since $y = \alpha(x)$ for some $x \in C_A$, 
the embedding property gives $x \leq_A g$, so 
$x \in \llbracket c \rrbracket_A$ and 
$y \in \alpha(\llbracket c \rrbracket_A)$.
($\supseteq$): Let $y = \alpha(x)$ for 
$x \in \llbracket c \rrbracket_A$. Then $x \leq_A g$ implies 
$\alpha(x) \leq_B \alpha(g)$ by preservation, so 
$y \in \llbracket \alpha(c) \rrbracket_{\mathcal{KB}'_B}$.
The proof extends to all operators whose denotation is defined via 
monotone closure over the KB ordering $\leq$.
\end{proof}

\begin{proposition}[Verdict Preservation]
\label{prop:alignment}
Let $\alpha: C_A \rightharpoonup C_B$ be an alignment 
(Definition~\ref{def:alignment}).
\begin{enumerate}
\item \textbf{Conflict preservation.} If 
  $\llbracket c_1 \rrbracket_A \cup \llbracket c_2 \rrbracket_A 
  \subseteq \text{dom}(\alpha)$ and 
  $\text{verdict}_A(c_1, c_2) = \textsc{Conflict}$, then 
  $\text{verdict}_{\mathcal{KB}'_B}(\alpha(c_1), \alpha(c_2)) = 
  \textsc{Conflict}$.
\item \textbf{Graceful degradation.} If 
  $\alpha(g_i) \neq \bot$ for the grounding values of $c_1, c_2$ 
  but $\llbracket c_i \rrbracket_A \not\subseteq 
  \text{dom}(\alpha)$ for some~$i$, then---without 
  downward-closure totality---the aligned verdict may 
  \emph{strengthen} to a false \textsc{Conflict} 
  (Example~\ref{ex:witness-loss}). With downward-closure totality 
  (Definition~\ref{def:alignment}), this case cannot arise for 
  grounded constraints: $\alpha(g) \neq \bot$ implies 
  ${\downarrow}\!g \subseteq \text{dom}(\alpha)$, so 
  $\llbracket c_i \rrbracket_A \subseteq \text{dom}(\alpha)$ 
  and the verdict reduces to case~(1) or the following: if 
  $\alpha(g_i) = \bot$ for some grounding value, then 
  $\alpha(c_i) = \top$ by 
  Definition~\ref{def:constraint-alignment}, yielding 
  \textsc{Unknown}. In no case does a false \textsc{Conflict} 
  arise.
\end{enumerate}
\end{proposition}

\begin{proof}
(1) \textsc{Conflict} implies 
$\llbracket c_1 \rrbracket_A \cap \llbracket c_2 \rrbracket_A = 
\emptyset$ with both sets classical (not $\top$). Injectivity 
preserves disjointness: 
$\alpha(\llbracket c_1 \rrbracket_A) \cap 
\alpha(\llbracket c_2 \rrbracket_A) = \emptyset$. By 
Lemma~\ref{lem:denotation-equality}, 
$\llbracket \alpha(c_1) \rrbracket_{\mathcal{KB}'_B} \cap 
\llbracket \alpha(c_2) \rrbracket_{\mathcal{KB}'_B} = \emptyset$, 
yielding \textsc{Conflict}.

(2) When $\alpha(g) = \bot$ for a grounding value, 
Definition~\ref{def:constraint-alignment} yields $\top$, and 
$\top \sqcap D = \top$ by Definition~\ref{def:intersection}, 
producing \textsc{Unknown}. When $\alpha(g_i) = \bot$ for a grounding value, 
$\alpha(c_i) = \top$ by Definition~\ref{def:constraint-alignment}, 
yielding \textsc{Unknown}. When denotations partially exceed 
$\operatorname{dom}(\alpha)$, witness completeness ensures this 
cannot occur for grounded constraints. In neither case can a false 
\textsc{Conflict} arise.
\end{proof}

The monotonicity guarantee in Proposition~\ref{prop:alignment}(2) is 
the central safety property: under witness completeness, alignment can 
only weaken verdicts toward \textsc{Unknown}, never fabricate conflicts 
(Example~\ref{ex:witness-loss}).

\begin{corollary}[Subsumption Preservation]
\label{cor:subsumption-preservation}
Let $\alpha: C_A \rightharpoonup C_B$ be an alignment. If 
$\llbracket c_1 \rrbracket_A \cup \llbracket c_2 \rrbracket_A 
\subseteq \operatorname{dom}(\alpha)$ and 
$c_1 \sqsubseteq_c c_2$ in $\mathcal{KB}_A$, then 
$\alpha(c_1) \sqsubseteq_c \alpha(c_2)$ in $\mathcal{KB}'_B$.
\end{corollary}
\begin{proof}
By Lemma~\ref{lem:denotation-equality}, 
$\llbracket \alpha(c_i) \rrbracket_{\mathcal{KB}'_B} = 
\alpha(\llbracket c_i \rrbracket_A)$ for $i = 1, 2$. Injectivity 
of $\alpha$ preserves subset inclusion: 
$\llbracket c_1 \rrbracket_A \subseteq 
\llbracket c_2 \rrbracket_A$ implies 
$\alpha(\llbracket c_1 \rrbracket_A) \subseteq 
\alpha(\llbracket c_2 \rrbracket_A)$.
\end{proof}

\begin{proposition}[KB Monotonicity]
\label{prop:kb-monotonicity}
Let $\mathcal{KB} = (C, \leq, \perp\!\!\!\perp, \gamma)$ and 
$\mathcal{KB}^+ = (C, \leq, \perp\!\!\!\perp^+, \gamma^+)$ with 
${\perp\!\!\!\perp} \subseteq {\perp\!\!\!\perp^+}$ and $\gamma^+$ 
extending $\gamma$ (same $C$ and $\leq$). Then:
\begin{enumerate}
\item If $\operatorname{verdict}_{\mathcal{KB}}(c_1, c_2) = 
  \textsc{Conflict}$, then 
  $\operatorname{verdict}_{\mathcal{KB}^+}(c_1, c_2) = 
  \textsc{Conflict}$.
\item If $\operatorname{verdict}_{\mathcal{KB}}(c_1, c_2) = 
  \textsc{Compatible}$, then 
  $\operatorname{verdict}_{\mathcal{KB}^+}(c_1, c_2) = 
  \textsc{Compatible}$.
\item If $\operatorname{verdict}_{\mathcal{KB}}(c_1, c_2) = 
  \textsc{Unknown}$, the verdict may resolve to any value.
\end{enumerate}
\end{proposition}
\begin{proof}
(1)~and~(2): Denotations depend on $C$ and $\leq$, both 
unchanged. Thus $\llbracket c_i \rrbracket_{\mathcal{KB}^+} = 
\llbracket c_i \rrbracket_{\mathcal{KB}}$ for all grounded 
constraints. Empty intersections remain empty; non-empty 
intersections remain non-empty. For previously ungrounded 
constraints ($\llbracket c \rrbracket = \top$), the original 
verdict was \textsc{Unknown}, so neither case~(1) nor~(2) 
arises.

(3)~Extending $\gamma$ replaces $\top$ with a classical set, 
enabling definite intersection and yielding either 
\textsc{Conflict} or \textsc{Compatible}.
\end{proof}

\subsubsection{Runtime Semantics}
\label{sec:runtime}
The denotational semantics reasons about \emph{concepts}---abstract 
KB entities like \texttt{gn:France} or \texttt{dpv:ScientificResearch}. 
At runtime, a policy engine receives concrete requests with actual 
values. We bridge this gap with two definitions and a soundness result.

\begin{definition}[Execution Context]
\label{def:context}
An execution context $\omega: \mathcal{L} \rightharpoonup \mathcal{V}$ 
is a partial function assigning right operand values to left operands, 
representing a concrete access request.
\end{definition}

\begin{definition}[Constraint Satisfaction]
\label{def:satisfaction}
Given $\mathcal{KB}_\ell = (C, \leq, \perp\!\!\!\perp, \gamma)$, 
context $\omega$ satisfies constraint $c = (\ell, \bowtie, v)$, 
written $\omega \models c$, if:
\begin{equation}
\ell \in \operatorname{dom}(\omega) \land 
\gamma(\omega(\ell)) \neq \bot \land 
\big(\llbracket c \rrbracket = \top \lor 
\gamma(\omega(\ell)) \in \llbracket c \rrbracket\big)
\end{equation}
When $\gamma(\omega(\ell)) = \bot$---the runtime value has no KB 
counterpart---the context does not satisfy the constraint. This 
default-deny stance at enforcement complements the conservative 
\textsc{Unknown} at analysis time: ungrounded values block neither 
policy authoring nor conflict detection, but do block runtime access 
until the KB is extended.
\end{definition}

\begin{theorem}[Runtime Soundness]
\label{thm:refined-soundness}
For constraints $c_1 = (\ell, \bowtie_1, v_1)$ and 
$c_2 = (\ell, \bowtie_2, v_2)$ over the same operand, 
if $\operatorname{verdict}(c_1, c_2) = \textsc{Conflict}$, then 
$\neg\exists\omega : (\omega \models c_1 \land \omega \models c_2)$.
\end{theorem}
\begin{proof}
Suppose $\omega \models c_1 \land \omega \models c_2$. 
By Definition~\ref{def:satisfaction}, 
$\ell \in \operatorname{dom}(\omega)$ and 
$\gamma(\omega(\ell)) \neq \bot$. 
Since $\operatorname{verdict}(c_1, c_2) = \textsc{Conflict}$, 
Definition~\ref{def:conflict} requires both denotations are grounded 
(not $\top$) with $\llbracket c_1 \rrbracket \cap 
\llbracket c_2 \rrbracket = \emptyset$.
Therefore $\gamma(\omega(\ell)) \in \llbracket c_1 \rrbracket$ 
and $\gamma(\omega(\ell)) \in \llbracket c_2 \rrbracket$, implying 
$\llbracket c_1 \rrbracket \cap \llbracket c_2 \rrbracket \neq 
\emptyset$---contradicting the conflict verdict.
\end{proof}

This result justifies \emph{static policy analysis}: policy 
administrators can reject incompatible constraint combinations during 
authoring, before any data flows. Soundness does not imply 
completeness: \textsc{Unknown} may conceal conflicts that only 
manifest at runtime when additional context resolves KB gaps. All 
meta-theorems (Theorems~\ref{thm:soundness}--\ref{thm:refined-soundness}, 
Lemmas~\ref{lem:disj-consistency}--\ref{lem:conflict-propagation}, 
Propositions~\ref{prop:decidability}--\ref{prop:kb-monotonicity}, 
Corollary~\ref{cor:subsumption-preservation}) are mechanically 
verified in Isabelle/HOL; the theory file is included in the 
supplementary material.

\subsection{Mechanical Verification}
\label{sec:isabelle}

All meta-theorems are mechanically verified in Isabelle/HOL 2025: The 
formalization uses three locales: \texttt{knowledge\_base} encodes Definition~\ref{def:kb}. 
  Carries 11 results including soundness 
  (Theorems~\ref{thm:soundness}--\ref{thm:composition-soundness}), 
  conflict propagation (Lemma~\ref{lem:conflict-propagation}), and 
  disjointness--order consistency 
  (Lemma~\ref{lem:disj-consistency}).
\texttt{kb\_extension} models KB evolution with extended 
  $\gamma$ and $\perp\!\!\!\perp$. Verifies KB monotonicity 
  (Proposition~\ref{prop:kb-monotonicity}).
\texttt{alignment} formalizes 
  Definition~\ref{def:alignment}. Verifies denotation, conflict, 
  and subsumption preservation 
  (Lemma~\ref{lem:denotation-equality}, 
  Proposition~\ref{prop:alignment}, 
  Corollary~\ref{cor:subsumption-preservation}).

\noindent All proofs complete via automated tactics 
(\texttt{simp}, \texttt{auto}, \texttt{blast}) without manual 
proof scripts---reflecting that definitions were chosen to 
remain within the decidable fragment Isabelle's simplifier 
handles natively. Table~\ref{tab:isabelle} gives the full 
correspondence; the theory file is in the supplementary material.

\begin{table}[t]
\centering
\caption{Paper results and Isabelle/HOL mechanization. \textbf{44 results}, 0 \texttt{sorry}.}
\label{tab:isabelle}
\small
\begin{tabular}{@{}lll@{}}
\toprule
\textbf{Paper Result} & \textbf{Isabelle Name} & \textbf{Locale} \\
\midrule
Lemma~\ref{lem:disj-consistency} & \texttt{disj\_order\_consistency} & \texttt{knowledge\_base} \\
Theorem~\ref{thm:soundness} & \texttt{soundness} & \texttt{knowledge\_base} \\
Theorem~\ref{thm:refined-soundness} & \texttt{runtime\_soundness} & \texttt{knowledge\_base} \\
Theorem~\ref{thm:composition-soundness} & \texttt{composition\_soundness\_*} & \texttt{knowledge\_base} \\
Lemma~\ref{lem:conflict-propagation} & \texttt{conflict\_propagation} & \texttt{knowledge\_base} \\
Proposition~\ref{prop:kb-monotonicity} & \texttt{kb\_monotonicity\_verdict} & \texttt{kb\_extension} \\
Lemma~\ref{lem:denotation-equality} & \texttt{denotation\_preservation\_isA} & \texttt{alignment} \\
Proposition~\ref{prop:alignment}(1) & \texttt{verdict\_preservation\_conflict\_full} & \texttt{alignment} \\
Corollary~\ref{cor:subsumption-preservation} & \texttt{subsumption\_preservation\_verdict} & \texttt{alignment} \\
\bottomrule
\end{tabular}
\end{table}
\section{Evaluation}
\label{sec:evaluation}
We validate the framework through 154 benchmarks across 
19~categories, evaluated by two independent provers: Vampire 
(superposition calculus) and Z3 (DPLL(T)). All problems are 
encoded in both TPTP and SMT-LIB, enabling dual-prover 
concordance as a correctness check. Both provers agree on all 
154~verdicts.
\subsection{Benchmark Design}
\label{sec:bench-design}
Each benchmark pairs two ODRL constraints over one or more 
operands, grounded against concrete knowledge bases. The 
negated-conjecture pattern (Table~\ref{tab:verdicts}) maps 
each verdict to a unique prover output. Sixteen Layer~0 
axiom files instantiate three semantic domains 
(Table~\ref{tab:kbs}).
\begin{table}[t]
\centering
\caption{Knowledge bases and axiom files used in the benchmark 
suite. Original KBs provide primary domain knowledge; 
alternative KBs test cross-standard alignment; auxiliary files 
supply UNA axioms and structural edge cases.}
\label{tab:kbs}
\small
\begin{tabular}{llcl}
\toprule
\textbf{Operand} & \textbf{KB} & \textbf{$|C|$} & \textbf{Role} \\
\midrule
\texttt{spatial}  & GEO000 &  4 & GeoNames (primary) \\
                  & GEO001 &  3 & ISO~3166 (alignment) \\
                  & GEO002 &  --- & UNA + negative \texttt{partOf} \\
\texttt{purpose}  & DPV000 & 10 & W3C DPV (primary) \\
                  & DPV001 &  6 & GDPR-derived (alignment) \\
                  & DPV002 &  --- & Same-level UNA \\
                  & DPV003 &  --- & Cross-level UNA \\
\texttt{language} & LNG000 & 10 & BCP~47 (primary) \\
                  & LNG001 &  6 & ISO~639-3 (alignment) \\
                  & LNG002 &  --- & UNA \\
\texttt{channel}  & NOM000 &  4 & Nominal (delivery channels) \\
\addlinespace[3pt]
\multicolumn{4}{l}{\textit{Structural KBs (operator stress tests)}} \\
                  & CHN000 &  5 & Depth-5 linear chain \\
                  & DIA000 &  4 & Diamond inheritance (DAG) \\
                  & SNG000 &  1 & Single-concept (domain closure) \\
                  & NMS000 &  6 & Near-miss overlap \\
\bottomrule
\end{tabular}
\end{table}
\subsection{Operator Coverage}
\label{sec:eval-operators}
The first 90~problems validate Definition~\ref{def:denotation} 
across all eight operators (\texttt{eq}, \texttt{neq}, 
\texttt{isA}, \texttt{isPartOf}, \texttt{hasPart}, 
\texttt{isAnyOf}, \texttt{isAllOf}, \texttt{isNoneOf}) and 
all three semantic domains (taxonomic, mereological, nominal).
Problems 010--056 cover the six original operators over 
GeoNames, DPV, and BCP~47. Problems 090--096 introduce 
\texttt{neq} (complement denotation $C \setminus \{g\}$), 
including the edge case where $|C|=1$ collapses the complement 
to $\emptyset$ (ODRL095). Problems 100--106 test 
\texttt{hasPart} (upward closure $\{x \mid g \leq x\}$), 
validating that it produces the inverse of \texttt{isPartOf}. 
Problems 110--147 extend \texttt{isAnyOf}, \texttt{isAllOf}, 
and \texttt{isNoneOf} coverage to mereological and nominal 
domains, and introduce 12~cross-operator pair tests (150--161) 
that stress interactions absent from single-operator benchmarks.
Three structural KBs target adversarial patterns: a depth-5 
chain validates transitive closure (170--171), a diamond DAG 
validates multiple-inheritance (172--174), and a single-concept 
KB validates degenerate cases where complements collapse 
(175--176). The near-miss KB (177--178) tests the boundary 
between \textsc{Compatible} and \textsc{Conflict} when 
denotations overlap on exactly one concept.
\paragraph{Nominal domain (140--147).}
Under identity ($\leq\,{=}\,{=}$), \texttt{isA} degenerates 
to \texttt{eq} and \texttt{isAllOf} requires all values to 
coincide. All eight nominal tests confirm this degeneration, 
validating the claim in Definition~\ref{def:kb} that operator 
semantics is determined by the KB relation, not the operator 
name.
\subsection{Logical Composition}
\label{sec:eval-composition}
Twenty-one problems validate Definition~\ref{def:composition} 
across all three ODRL composition modes.
Conjunction (\texttt{and}, 030--033 + 200--201): 
multi-operand \texttt{and} requires all dimensions to overlap 
simultaneously. A single \textsc{Conflict} or \textsc{Unknown} 
blocks the conjunction. ODRL200 confirms three-operand 
compatibility; ODRL201 confirms that a language conflict in 
the third operand blocks an otherwise-compatible pair.
Disjunction (\texttt{or}, 080--084 + 202--203): at least one 
branch must overlap. ODRL202 confirms that a single 
\textsc{Compatible} branch resolves the disjunction even when 
the other is \textsc{Unknown}. ODRL206 tests nested 
spatial~$\wedge$~or(purpose). Exclusive disjunction (\texttt{xone}, 085--088 + 204--207): 
the key finding is that \texttt{xone} requires strictly 
stronger KB axioms than conjunction or disjunction. 
Table~\ref{tab:xone-comparison} shows that asymmetric 
negative-axiom coverage in the DPV KB produces different 
verdicts for structurally symmetric problems.
By design, composition operators are evaluated at the 
meta-level: the prover returns per-branch verdicts, and 
Definition~\ref{def:composition} combines them via the 
truth table. This decomposition is essential---encoding 
\texttt{xone} directly would require $\exists$ under 
negation ($\forall\exists\forall$ prefix), breaking the 
EPR fragment and forfeiting the decidability guarantee.
\begin{table}[t]
\centering
\caption{XONE requires explicit negative axioms. Both problems 
test \texttt{xone(commercial, nonCommercial)} but with different 
request values. The asymmetric KB coverage produces different 
verdicts.}
\label{tab:xone-comparison}
\small
\begin{tabular}{llccc}
\toprule
\textbf{Problem} & \textbf{Request} & \textbf{$\sqsubseteq$ branch?} 
  & \textbf{$\lnot\sqsubseteq$ other?} & \textbf{Verdict} \\
\midrule
ODRL085 & nonCommRes & \checkmark & explicit & \textsc{Compatible} \\
ODRL086 & commRes    & \checkmark & absent   & \textsc{Unknown} \\
\bottomrule
\end{tabular}
\end{table}
\subsection{Cross-KB Alignment}
\label{sec:eval-alignment}
Twenty-three problems validate 
Proposition~\ref{prop:alignment}. The original 13~problems 
(057--066) test three alignment pairs with complete concept 
coverage. The extension (190--199) targets adversarial edge 
cases: unmapped witnesses that force degradation (190--192), 
operator coverage across aligned KBs (\texttt{isAnyOf}, 
\texttt{isNoneOf}, \texttt{hasPart}, \texttt{neq}: 193--197), 
and multi-domain alignment where all three operands degrade 
simultaneously (199).
Across all 23~tests, no false \textsc{Conflict} arises. 
Degradation tests include paired checks: for each case, both 
the compatibility and conflict conjectures return 
\textsc{CounterSatisfiable}, confirming that the verdict is 
genuinely indeterminate rather than accidentally resolved.
\subsection{Runtime Semantics}
\label{sec:eval-runtime}
Six problems (070--075) bridge concept-level reasoning to 
runtime enforcement 
(Theorem~\ref{thm:refined-soundness}): witness extraction for 
\textsc{Compatible} pairs, pointwise rejection for 
\textsc{Conflict} pairs, and an exhaustive finite-model check 
(ODRL073) confirming that no BCP~47 concept satisfies both 
constraints in a conflict pair.

\begin{remark}[Branch-Aware Exclusivity]
\label{rem:branch-xone}
Definition~\ref{def:composition} conservatively requires all 
cross-branch per-operand verdicts to be \textsc{Conflict} for 
\texttt{xone} compatibility. When branches use conjunction 
internally, this is stricter than necessary: a single 
per-operand \textsc{Conflict} suffices for mutual exclusivity 
by Theorem~\ref{thm:composition-soundness}, since one failing 
operand under \texttt{and} renders the entire branch 
unsatisfiable. The refined exclusivity test applies 
$\operatorname{verdict}_{\texttt{and}}$ to cross-branch 
per-operand verdicts:
\begin{equation}
\operatorname{excl}(B_j, B_k) = 
  \operatorname{verdict}_{\texttt{and}}\big(
  \operatorname{verdict}(c_j^{(\ell)}, c_k^{(\ell)}) 
  \mid \ell \in \operatorname{shared}(B_j, B_k)\big)
\end{equation}
where $\operatorname{shared}(B_j, B_k)$ denotes operands 
constrained in both branches. When branches share no operands, 
$\operatorname{excl}(B_j, B_k) = \textsc{Compatible}$ (vacuous 
conjunction), correctly reflecting that orthogonal branches 
cannot be proven exclusive. This refinement strictly reduces 
\textsc{Unknown} verdicts without affecting soundness: the 
proof chains through Theorem~\ref{thm:soundness}, 
Assumption~\ref{assm:operand-independence}, and 
Theorem~\ref{thm:composition-soundness}.
\end{remark}

\paragraph{Related Work.}
Concurrent standardization efforts address ODRL processor 
conformance~\cite{rodriguez2025conformance}, complementing our 
focus on constraint-level semantics. A comprehensive discussion 
of related work will appear in the extended version.

\section{Conclusion}
\label{sec:conclusion}
We presented a semantic grounding framework for ODRL constraint 
conflict detection, validated by a benchmark suite of 
154~problems across 19~categories with 100\% agreement 
between the Vampire theorem prover and the Z3 SMT solver~\cite{vampire2013,z3}.
The framework makes three contributions. First, a denotational 
semantics for all eight ODRL operators with a 
three-valued verdict (\textsc{Conflict}, \textsc{Compatible}, 
\textsc{Unknown}) that is sound under incomplete knowledge. 
Second, a decidable EPR encoding with bidirectional denotation 
rules and formal composition semantics for ODRL's \texttt{and}, 
\texttt{or}, and \texttt{xone} operators---revealing that 
exclusive composition requires strictly stronger KB axioms 
than conjunction or disjunction. Third, cross-KB alignment 
with proven verdict preservation and graceful degradation, 
bridging the gap between dataspaces that adopt different 
standards for the same operand.
The benchmark suite covers single-KB reasoning (90~problems), 
logical composition including all three ODRL modes 
(21~problems), cross-standard alignment with adversarial 
edge cases (23~problems), and runtime soundness validation 
(6~problems). The alignment and composition results---absent 
from all prior ODRL conflict detection work---demonstrate that 
semantic grounding enables interoperability across dataspaces 
without manual re-encoding and without introducing false 
conflicts.\paragraph{Future work.}
Three directions remain: extending coverage to the remaining 
11~KB-dependent operands, rule-level conflict detection composing 
constraint-level verdicts with permission/prohibition resolution, 
and a standardized KB registry enabling dataspaces to discover 
and reuse verified alignments.
\bibliographystyle{splncs04}
\bibliography{references}
\end{document}